%% file: sc_working.tex
\newif\ifdraft \drafttrue
\newif\iffull \fulltrue

\newcommand{\dong}[1]{\textcolor{blue}{#1}} 

\draftfalse\fulltrue

\documentclass{article}
\usepackage[utf8]{inputenc}
\usepackage{amssymb,amsmath,amsthm,amsfonts}
\usepackage{amsfonts}
\usepackage[usenames,dvipsnames]{color}
\usepackage{algorithm}
\usepackage[noend]{algorithmic}
\usepackage{natbib}
\usepackage{cleveref}

\newcommand{\INDSTATE}[1][1]{\STATE\hspace{#1\algorithmicindent}}

\usepackage{comment}
\usepackage{fullpage}
\newtheorem{claim}{Claim}

\newtheorem{lemma}{Lemma}
\newtheorem{theorem}{Theorem}

\newtheorem*{claim-non}{Claim}
\newtheorem*{theorem-non}{Theorem}

\newtheorem{corollary}{Corollary}[theorem]

\newcommand{\sw}[1]{\ifdraft \textcolor{blue}{[Steven: #1]}\fi}
\newcommand{\ar}[1]{\ifdraft \textcolor{blue}{[Aaron: #1]}\fi}
\newcommand{\zs}[1]{\ifdraft \textcolor{ForestGreen}{[ Zach: #1]}\fi}
\newcommand{\djs}[1]{\ifdraft \textcolor{red}{[Jinshuo: #1]}\fi}

\newcommand{\xb}{\langle \hat{x}, \beta \rangle}
\newcommand{\xbb}{\langle \hat{x}(\beta), \beta \rangle}

\newcommand{\Rd}{\mathbb{R}^d}

\newcommand{\innprod}[2]{\langle{#1},{#2}\rangle}

\theoremstyle{definition}
\newtheorem{definition}{Definition}
\theoremstyle{remark}
\newtheorem{remark}{Remark}

\DeclareMathOperator*{\argmax}{arg\,max}
\DeclareMathOperator*{\argmin}{arg\,min}

\title{Strategic Classification from Revealed Preferences}
\author{Jinshuo Dong \and Aaron Roth \and Zachary Schutzman \and Bo Waggoner \and Zhiwei Steven Wu}
\begin{document}
\maketitle

\input{abs}

\input{intro}

\input{prelims}

\input{convex}

\input{opt}
\input{future}

\bibliography{mybib}
\bibliographystyle{apalike}

\appendix

\input{const}
\input{appendix}

\end{document}

%% file: abs.tex
\begin{abstract}
We study an online linear classification problem, in which the data is generated by strategic agents who manipulate their features in an effort to change the classification outcome. In rounds, the learner deploys a classifier, and an adversarially chosen agent arrives, possibly manipulating her features to optimally respond to the learner. The learner has no knowledge of the agents' utility functions or ``real'' features, which may vary widely across agents. Instead, the learner is only able to observe their ``revealed preferences'' --- i.e. the actual manipulated feature vectors they provide. For a broad family of agent cost functions, we give a computationally efficient learning algorithm that is able to obtain diminishing ``Stackelberg regret'' --- a form of policy regret that guarantees that the learner is obtaining loss nearly as small as that of the best classifier in hindsight, even allowing for the fact that agents will best-respond differently\djs{This word is a little confusing for me...} to the optimal classifier.

\end{abstract} 

%% file: intro.tex
\section{Introduction}

Machine learning is typically studied under the assumption that the data distribution a classifier is deployed on is the same as the data distribution it was trained on. However, the outputs of many classification and regression problems are used to make decisions about human beings, such as whether an individual will receive a loan, be hired, be admitted to college, or whether their email will pass through a spam filter. In these settings, the individuals have a vested interest in the outcome, and so the data generating process is better modeled as part of a strategic game in which individuals edit their data to increase the likelihood of a certain outcome. Tax evaders may carefully craft their tax returns to decrease the likelihood of an audit. Home buyers may strategically sign up for more credit cards in an effort to increase their credit score. Email spammers may modify their emails in order to evade existing filters. In each of these settings, the individuals have a natural objective that they want to maximize --- they want to increase their probability of being (say) positively classified. However, they also experience a cost from performing these manipulations (tax evaders may have to pay \emph{some} tax to avoid an audit, and email spammers must balance their ability to evade spam filters with their original goal in crafting email text). These costs can be naturally modeled as the distance between the ``true'' features $x$ of the individual and the manipulated features $x'$ that he ends up sending, according to some measure. In settings of this sort, learning can be viewed as a game between a learner and the set of individuals who generate the data, and the goal of the learner is to compute an equilibrium strategy of the game (according to some notion of equilibrium) that maximizes her utility.

The relevant notion of equilibrium depends on the order of information revelation in the game. Frequently, the learner will first deploy her classifier, and then the data generating players (agents) will get to craft their data with knowledge of the learner's classifier. In a setting like this, the learner should seek to play a \emph{Stackelberg equilibrium} of the game --- i.e. she should deploy the classifier that minimizes her error \emph{after} the agents are given an opportunity to best respond to the learner's classifier. This is the approach taken by the most closely related prior work:
\cite{BS11} and
\cite{HMPW16}. Both of these papers consider a one-shot game and study how to compute the Stackelberg equilibria of this game. To do this, they necessarily assume that the learner has (almost) full knowledge of the agents' utility functions; in particular, it is assumed that the learner has access to the ``true'' distribution of agent features (before manipulation), and that the costs experienced by the agents for manipulating their data are the same for all agents and known to the learner\footnote{The particulars of the models studied in \citet{BS11} and \citet{HMPW16} differ. Br{\"u}ckner and Scheffer model a single data generation player who manipulates the data distribution, and experiences cost equal to the squared $\ell_2$ distance of his manipulation. Hardt et al. study a model in which each agent can independently manipulate his own data point, but assume that all agents experience cost as a function of the same separable cost function, known to the learner.}.

The primary point of departure in our work is to study the strategic classification problem when the learner does not know the utility functions of the agents: neither their true features $x$, nor the cost they experience for manipulation (which can now differ for each agent). In this setting, it no longer makes sense to study learning as a one-shot game; we cannot compute an equilibrium when the utility functions of the agents are unknown. Instead, we model the learning process as an iterative, online procedure. In rounds $t$, the learner proposes a classifier $\beta_t$. Then, an agent arrives, with an unknown set of \dong{$d$} ``true'' features $x_t \in \mathbb{R}^d$, and an unknown cost function $d_t:\mathbb{R}^d \times \mathbb{R}^d\rightarrow \mathbb{R}$ for manipulation. The learner observes only the manipulated set of features $\hat{x}_t \in \mathbb{R}^d$, that represent the agent's best response to $\beta_t$. After classification, the learner observes the agent's true label $y_t$, and suffers the corresponding loss. Crucially, the learner never gets to observe either $x_t$ or $d_t$, and we do not even assume that these are drawn from a distribution (they may be adversarially chosen). The only access that the learner has to these parameters is via the \emph{revealed preferences} of the agents --- the learner gets to observe the actions of the agents, which are the result of optimizing their unknown utility functions. The learner must use this information to improve her classifier.

We measure the performance of our algorithms via a quantity that we call \emph{Stackelberg regret}: informally, by comparing the average loss of the learner to the loss she could have experienced with the best fixed classifier in hindsight, \emph{taking into account that the agents would have best-responded differently had she used a different classifier}. If the learner were in fact interacting with the same agent repeatedly, or if the agents $(x_t,d_t)$ were drawn from a fixed distribution, then the guarantee of diminishing Stackelberg regret would imply the convergence to a Stackelberg equilibrium of the corresponding one-shot game. \djs{make this claim?} However, Stackelberg regret is more general, and applies even to settings in which the agents are adversarially chosen.

We add one further twist. \djs{Do we need citations here?} Previous work on strategic classification has typically assumed that all agents are strategic. However, the equilibrium solutions that result from this assumption may be undesirable. For example, in a spam classification setting, the Stackelberg-optimal classifier may attain its optimal accuracy only if all agents --- even legitimate (non-spam) senders --- actively seek to manipulate their emails to avoid the spam filter. In these settings, it would be more desirable to compute a classifier that was optimal under the assumption that spammers would attempt to manipulate their emails in order to game the classifier, but that did not assume legitimate senders would. To capture this nuance, in our model, only agents whose true label $y_t = -1$ (e.g. spammers) are strategic, and agents for whom $y_t = 1$ are non-strategic.
\subsection{Our Results and Techniques}
The problem that the learner must solve is a bi-level optimization problem in which the objective of the inner layer (the agents' maximization problem) is unknown. Even with full information, bi-level optimization problems are often NP-hard. As a first step in our solution, we seek to identify conditions under which the learner's optimization problem is convex. Under these conditions, computing an optimal solution would be tractable in the full information setting. The remaining difficulty is solving the optimization problem in our limited feedback model.

We study learners who deploy \emph{linear classifiers} $\beta_t$, and consider two natural learner loss functions: logistic loss (corresponding to logistic regression) and hinge loss (corresponding to support vector machines). The agents in our model are parameterized by target feature vectors $x_t$ and cost functions $d_t$, and will play the modified feature vector $\hat{x_t}(\beta_t)$ that maximizes their utility given $\beta_t$. We model agent utility functions as $u_t(\hat{x}, \beta) = \langle \beta, \hat{x} \rangle - d_t(x_t, \hat{x}_t)$. Using tools from convex analysis, we give general conditions on the cost functions $d_t$ that suffice to make the learner's objective convex, for both logistic and hinge loss, for all $x_t$. These conditions are satisfied by (among other classes of cost functions) any squared Mahalanobis distance and, more generally, by any norm-induced metric raised to a power greater than one.

Finally, we turn to the learner's optimization problem. Once we have derived conditions under which the learner's optimization problem is convex, we can in principle achieve quickly diminishing Stackelberg regret with any algorithm for online bandit (i.e. zeroth order) optimization that works for adversarialy chosen loss functions. However, we observe that when some of the agents are non-strategic (e.g. the non-spammers), there is additional structure that we can take advantage of.
In particular, on rounds for which the agent is non-strategic, the learner can also derive gradients for her loss function, in contrast to rounds on which the agent is strategic, where she only has access to zeroth-order feedback.
To take advantage of this, we analyze a variant of the bandit convex optimization algorithm of
\cite{flaxman2005online} which can make use of both kinds of feedback. The regret bound we obtain interpolates between the bound of \cite{flaxman2005online}, obtained when all agents are strategic, and the regret bound of online gradient descent \cite{Zin03}, obtained when no agents are strategic, as a function of the proportion of the observed agents which were strategic.
\subsection{Further Related Work}
In the \emph{adversarial} or \emph{strategic} learning literature, the most closely related works are \cite{BS11} and \cite{HMPW16}, as discussed above, which also consider notions of Stackelberg equilibrium, and make other similar modelling choices.
\cite{LC09} also model adversarial learning as a Stackelberg game. Other works in this line model the learning problem as a purely adversarial (zero sum) game as in \cite{DDSV04} for which the appropriate solution concept is a minmax equilibrium, study the simultaneous move equilibria of non-zero sum games as in \cite{BKS12, BS09}, and study the Bayes-Nash equilibria of incomplete information games given Bayesian priors on the player types as in \cite{GSBS13}. Common to all of these works is the assumption that the learner either has full knowledge of the data generator's utility functions (when Nash and Stackelberg equilibria are computed), or else knowledge of a prior distribution (when Bayes-Nash equilibria are computed). The point of departure of the current paper is to assume that the learner does \emph{not} have this knowledge, and instead only has the power to observe agent decisions in response to deployed learners.

There is a parallel literature in machine learning and algorithmic game theory focusing on the problem of \emph{learning from revealed preferences} --- which corresponds to learning from the choices that agents make when optimizing their (unknown) utility functions in response to some decision of the learner as in \cite{BV06}. This literature has primarily focused on how buyers with unknown combinatorial valuation functions make purchases in response to prices. Learning problems studied in this literature include learning to predict what purchase decisions a buyer will make in response to a set of prices drawn from an unknown distribution as in \cite{ZR12,BDMUV14}, finding prices that will maximize profit or welfare after buyers best respond as in \cite{ACDKR15,RUW16,RSUW17}, and generalizations of these problems as in \cite{JRRW16}. We study the problem of strategic classification with this sort of ``revealed preferences'' feedback.

Finally, Stackelberg games are studied extensively in the ``security games'' literature: see \cite{tambe2011security} for an overview. Most closely related to this paper is the work of \cite{security}, who develop no-regret algorithms for certain kinds of security games when ``attackers'' arrive online, using a notion of regret that is equivalent to the ``Stackelberg regret'' that we bound. This work is similar in motivation to ours: its goal is to give algorithms to compute equilibrium-like strategies for the ``defender'' without assuming that he has an unrealistic amount of knowledge about his opponents' utility functions. Technically, the work is quite different, since we are operating in very different settings. In particular, \cite{security} are primarily interested in information theoretic bounds, and do not give computationally efficient algorithms (since in general, solving Stackelberg security games is NP hard --- see \cite{korzhyk2010complexity}.)

%% file: prelims.tex
\section{Model and Preliminaries}

We study a sequential binary classification problem in which a learner
wishes to classify a sequence of examples over $n$ rounds. The example
at each round $t$ is associated with an agent $a_t=(x_t,y_t,d_t)$,
where $x_t \in \mathbb{R}^d$ is the feature vector, $y_t \in \{-1,1\}$
is the true label and
$d_t:\mathbb{R}^d\times \mathbb{R}^d\rightarrow \mathbb{R}$ is a
distance function that maps pairs of feature vectors to costs.  We will think of $d_t(x,\hat{x})$ as the cost for the agent to change his feature vector $x$ to the feature vector $\hat{x}$. If the
example is positive ($y_t = 1$), then we say the agent is
\emph{non-strategic}, and if the example is negative ($y_t = -1$),
then we say the agent is \emph{strategic}.


Each agent $a_t$ has a utility function
$u_t(x, \beta) = \innprod{\beta}{x} - d_t(x_t,x)$. In each round $t$,
the interaction between the learner and the agent is the following:

\begin{enumerate}
    \item The learner commits to a linear classifier parameterized by $\beta_t \in K$, where $K$ is the set of feasible parameters.	

	\item An adversary, oblivious to the learner's choices up to round $t$, selects an agent $a_t=(x_t,y_t,d_t)$.
		
	\item The agent $a_t$ sends the data point $\hat{x}_t$ to the learner:
	\begin{itemize}
        \item If the agent is strategic ($y_t=-1$), then the agent
          sends the learner his best response to $\beta_t$ (ties
          broken arbitrarily):
          \begin{equation}\label{br}
            \hat{x}_t(\beta_t) \in \arg\max_x u_t(x,\beta)
          \end{equation}

        \item If the agent is non-strategic ($y_t=1$), then the agent
          does not modify the feature vector, and so sends
          $\hat{x}_t(\beta_t) = x_t$ to the learner.
	\end{itemize}

      \item The learner observes $y_t$, and experiences classification
        loss $c_t(\beta_t) = c(\hat{x}_t,y_t,\beta_t)$.
\end{enumerate}

We are mainly interested in two standard classification loss functions
$c$.  The first is logistic loss, which corresponds to logistic
regression:
$$c_{log}(\hat{x}_t,y_t,\beta_t) = \log(1+e^{-y_t\cdot \innprod{\hat{x}_t}{\beta_t}})$$
The second is hinge loss, which corresponds to a support vector machine:
$$c_h(\hat{x}_t,y_t,\beta_t) = \max(0,1-y_t\cdot\innprod{\hat{x}_t}{\beta_t})$$

The interaction between the learner and the agent in each strategic
round can be viewed as a Stackelberg game, in which the learner as the
leader plays her strategy first, and then the agent as the follower
best responds. With this observation, we define a regret notion termed
as \textit{Stackelberg regret} for measuring the performance of the
learner.\footnote{The same regret notion has also appeared
  in the context of repeated security games \citep{security}.} In words, Stackelberg regret is the difference
between the cumulative loss of the learner and the cumulative loss it would
have experienced if it had deployed the single best classifier in
hindsight $\beta^*$, for the same sequence of agents (who would have
responded differently). More formally, for a history of play involving
the $n$ agents $A=(a_1,a_2,\dots,a_n)$ and the sequence of $n$ classifiers
$B=(\beta_1,\beta_2,\dots,\beta_n)$ the Stackelberg regret is defined
to be

$$\mathcal{R}_S(A,B) = \sum\limits_{t=1}^n c(\hat{x}_t(\beta_t),y_t,\beta_t) - \min_{\beta \in K}\sum\limits_{t=1}^{n} c(\hat{x}_t(\beta),y_t,\beta) $$

Observe that the regret-minimizing classifier $\beta^*$ is a
Stackelberg equilibrium strategy for the learner in the one-shot game
in which the learner first commits to a classifier $\beta^*$, then all
the agents $(a_1,\dots,a_n)$ simultaneously respond with
$\hat{x}_t(\beta^*)$, which results in the learner experiencing classification loss equal to $\sum\limits_{i=1}^n c(\hat{x}_i(\beta^*),y_i,\beta^*)$. 

In order to derive efficient algorithms with sub-linear Stackelberg
regret, we will impose some restrictions the agents' distance
functions $d_t$ (formally described in \Cref{sec:convex}), and we will
assume the feature vectors $x_t$ have norm bounded by $R_1$.  We also assume our feasible set
$K \subseteq \mathbb{R}^d$ is convex, contains the unit
$\ell_2$-ball, and has norm bounded by $R_2$, such that for any $\beta\in K$,
$\|\beta\|_2\leq R_2$.

%% file: convex.tex
\section{Conditions for a Convex Learning Problem}\label{sec:convex}

%

In this section, we derive general conditions under which the learner's problem in our setting can be cast as a convex minimization problem.
At each round, the learner proposes some hypothesis $\beta$ and receives a loss that depends on the best response of the strategic agent to $\beta$.
Even when the original loss function is a convex function of $\beta$ alone, holding $\hat{x},y$ fixed, it may no longer be convex when $\hat{x}$ is chosen strategically as a function of $\beta$.

Since the learner's objective is a summation of the loss over all rounds, and the sums of convex loss functions are convex, it suffices to show convexity for the loss experienced at each fixed round $t$. In the following, we omit the subscript $t$ to avoid notational clutter.

Recall that each strategic agent $a = (x,y,d)$ (with $y=-1$), when facing a classifier $\beta$, will misreport his feature vector $x$
as $\hat{x}(\beta)$ by solving a maximization problem:
\[
\hat{x}(\beta) \in \argmax_{\hat{x}\in \Rd} u(\hat{x},x,\beta) =\argmax_{\hat{x}\in \Rd} \xb - d(\hat{x},x)
\]
where $d(\cdot,\cdot)$ is a ``distance'' function modeling the cost of deviating from the intended $x$. 

We first show that both the logistic and hinge loss objective are convex if $\xbb$ is convex in $\beta$, and then prove the convexity of $\xbb$ when the distance has the form $d(x,y) = f(x-y)$, for $f$ satisfying reasonable assumptions (stated in Theorem \ref{thm:convex}). In \Cref{sec:example}, we give a large class of examples satisfying our assumptions, as well as several examples of functions which fail to satisfy our assumptions for which the problem we are studying is ill-posed (for example because strategic agents might not have finite best responses). This shows that it is necessary to impose constraints on $d$ when studying this problem.

As the first step, consider the learner's objective function,
$$c(x,y,\beta) = h(y\xbb)=h(y\langle \hat{x}(\beta,x,y), \beta\rangle).$$
where
\[
\hat{x}(\beta,x,y)=
\left\{
\begin{array}{ll}
x, 		& y=1, \\
\argmax u(\hat{x},x,\beta), & y=-1.
\end{array}
\right.
\]
Note that we are now writing $\hat{x}(\beta,x,y)$ rather than $\hat{x}(\beta)$ to make explicit the dependence on $x$ and $y$.
Where $x$ and $y$ are clear from context and fixed, we continue to write $\hat{x}(\beta)$. We note also that $\hat{x}(\beta,x,y)$ is not necessarily a well-defined mapping since the strategic agents can break the tie arbitrarily when the $\arg\max$ is not unique. However, as we show in \Cref{thm:convex}, $\xbb$ is well-defined, which is all that is necessary for the learner's cost to be well defined.

We instantiate $h(z)$ as either logistic loss $h(z)=\log(1 + e^{-z})$ or hinge loss $h(z) = (1-z)_+$\footnote{We write $(x)_+$ to denote $\max(x,0)$.}: note that both are non-increasing and convex functions in $z$.

We will rely on the following fact in convex analysis.

\begin{lemma}[e.g.~\cite{rockafellar2015convex}, Theorem 5.1] \label{lemma:increasing-convex-convex}
  Let $g:\mathbb{R}\to\mathbb{R}$ be a non-increasing convex function
  of a single variable and $F:\Rd\to \mathbb{R}$ be convex. Then
  $g(-F(x))$ is convex in $x$.
\end{lemma}

This yields the following approach.
\begin{theorem} \label{thm:sufficient-convex}
  Suppose, for all $(x,y)$, a strategic agent's best-response function $\hat{x}(\beta,x,y)$ satisfies the condition that the function $\beta \mapsto \xbb$ is convex.
  Then logistic loss and hinge loss are convex functions of $\beta$, i.e. for all $(x,y)$, both of the following are convex:
  \begin{itemize}
    \item (logistic) $c(x,y,\beta) = \log\left(1 + e^{-y \xbb}\right)$;
    \item (hinge) $c(x,y,\beta) = \left(1 - y\xbb \right)_+$.
  \end{itemize}
\end{theorem}
\begin{proof}
  For strategic agents, i.e. when $y=-1$, this follows immediately from Lemma \ref{lemma:increasing-convex-convex} and the fact that both loss functions are non-increasing convex functions of the variable $z = \xbb$.
  When $y=1$, i.e. the agent is not strategic, $c(x,y,\beta) = h(\langle \beta, x\rangle) = h(-\langle \beta, -x\rangle)$. Note that $\langle \beta, -x\rangle$ is linear (and hence convex), so the lemma applies here as well.
\end{proof}

Therefore, in order to show the convexity of the learner's loss function $c(x,y,\beta)$ for every $x,y$, it suffices to show that for any fixed $x$, $\xbb = \langle \hat{x}(\beta,x,y), \beta\rangle$ is convex in $\beta$. In the next section, we give sufficient conditions for this to be the case.

\subsection{Sufficient Conditions}
We need to recall a definition before we present the main result of this section.
\definition
{
A function $f:\mathbb{R}^d\rightarrow \mathbb{R}$ is \textit{positive homogeneous of degree $k$} if for any scalar $\alpha\in\mathbb{R}$ with $\alpha>0$ and vector $x\in\mathbb{R}^d$, we have
$$f(\alpha x) = \alpha^kf(x).$$
}
Note that $k$ need not be an integer.

\begin{theorem}\label{thm:convex}
	Let $u(\hat{x},x,\beta) = \innprod{\hat{x}}{\beta} - d(\hat{x},x)$ be the strategic agent's utility function.  If $d$ has the form $d(\hat{x},x) = f(\hat{x}-x)$ where $f:\Rd\to \mathbb{R}$ satisfies:
	\begin{itemize}
		\item $f(x)>0$ for all $x\neq0$;
		\item $f$ is convex over $\Rd$;
		\item $f$ is positive homogeneous of degree $r>1$
	\end{itemize}
	then the function $\beta \mapsto \langle \beta, \argmax\limits_{\hat{x}} u(\hat{x},x,\beta) \rangle$ is well-defined (i.e. finite and independent of the choices of maximizer) and convex.
\end{theorem}
Note that by this theorem, when the conditions hold we can speak of the function $\beta \mapsto \xbb$ without ambiguity even when there are multiple best responses $\hat{x}(\beta)$.

\begin{proof}
  The proof is broken up into a series of steps, which are described here.
  The remainder of this subsection then states each step in more detail and proves it.
  \begin{enumerate}
  	\item First, we show (Corollary \ref{cor:best}) that any best response $\hat{x}(\beta)$ satisfies $\hat{x}(\beta) = x + v$, where $v$ is some subgradient of $f^*$ at $\beta$ and $f^*$ is the convex conjugate of $f$.
  	\item Next, we show (Claim \ref{claim:finite}) that $f^*(\beta)$ is finite for all $\beta$.
  	\item Next, we prove (Claim \ref{claim:homo}) that homogeneity of $f^*$ follows from homogeneity of $f$.
  	\item Finally, we apply a slight generalization of Euler's Theorem, which says that if $f^*$ is homogenous and convex, there is a $k$ such that for any choice of subgradient $v \in \partial f^*(\beta)$ (the set of subgradients at $\beta$), we have $\langle v,\beta \rangle = kf^*(\beta)$.
          This, together with steps 2 and 3, implies that $\xbb = \langle x, \beta\rangle + kf^*(\beta)$ for any choice of best response $\hat{x}(\beta)$, and this function is well-defined and convex.
  \end{enumerate}
\end{proof}

All missing proofs appear in the appendix.



We begin with the first step. First, we rewrite the utility function:
$$u(\hat{x},x,\beta) = \innprod{\hat{x}}{\beta} - f(\hat{x}-x).$$
We perform a change of variable from $\hat{x}$ to $w = \hat{x} - x$ and write:

 $$v(w,\beta) = u(\hat{x},x,\beta) = \innprod{x}{\beta} + \innprod{w}{\beta} -f(w). $$

The first step only relies on the convexity of $f$.  Let
$$w^*(\beta) = \argmax_{w\in \Rd} v(w,\beta)$$
 be the set of maximizers of $v$ given $\beta$. We note that $w^*(\beta)$ is a convex set since $v(w,\beta)$ is concave in $w$ for any $\beta$. Note that if $f$ is differentiable, $w^*(\beta)$ is a singleton set.

\definition{The \textit{convex conjugate} of a function $f:\mathbb{R}^d\rightarrow \mathbb{R}$ , denoted $f^*:\mathbb{R}^d\rightarrow \mathbb{R}$ is the function defined as
	$$f^*(x^*) = \sup_{x\in \Rd} (\innprod{x^*}{x} - f(x)). $$
$f^*$ is always convex (even when $f$ is not) but is not necessarily finite. However, our positivity and homogeneity assumptions imply that it must be finite everywhere, as we will prove in step 2.

\sw{FIXED:still need to insert condition such that $y^*(\beta)$ is well-defined}
\begin{claim}
	The set $w^*(\beta)$ is equal to $\partial f^*(\beta)$, the set of subgradients of $f^*$ at $\beta$.
\end{claim}

\begin{proof}
	Fix a $\beta$. Recall that $$v(w,\beta) = \innprod{x}{\beta} + \innprod{w}{\beta} -f(w).$$
	$\innprod{x}{\beta}$ is a constant since $\beta$ is fixed and hence can be ignored. We have that
	\begin{align*}
   	  w\in w^*(\beta) &\Leftrightarrow 0 \in \partial[\langle w, \beta \rangle - f(w)] \\
	                  &\Leftrightarrow \beta \in \partial f(w)  \\
                      &\Leftrightarrow w \in \partial f^*(\beta).
	\end{align*}
	The first equivalence follows from the optimality of $w$. The second follows from computation. To show the last equivalence, we apply the following theorem from convex analysis:
	\begin{theorem}[\citet{rockafellar2015convex}, Theorem 23.5; cf Corollary 10.1.1, Theorem 7.1]
	Let $f:\Rd \to \mathbb{R}$ be convex and finite everywhere. Then
	\[
	x^* \in \partial f(x) \Leftrightarrow x\in \partial f^*(x^*).
	\]
	\end{theorem}
\end{proof}

Therefore, we have the following
\begin{corollary}\label{cor:best}
The Strategic agent's best-response set is
$$\argmax_{\hat{x}\in \Rd} u(\hat{x},x,\beta)=x+w^*(\beta)=x+\partial f^*(\beta).$$
\end{corollary}
\zs{FIXED:This is a little sloppy: we've been thinking of $\hat{x}(\beta)$ as a single (vector) value, but here we're treating it like a set...}


We now move on to step 2.  We want to show that $f^*$ is always finite and that the set of best responses of the strategic agent is non-empty and bounded.
\begin{claim}\label{claim:finite}
Under the assumptions in Theorem \ref{thm:convex}, $f^*(\beta)$ is finite for all $\beta\in\Rd$.
\end{claim}

See \Cref{pf:{claim:finite}} for the proof.

We now turn our attention to the set of best responses. The following theorem guarantees that the strategic agents' best response correspondence is well behaved:

\begin{theorem}[\citet{rockafellar2015convex}, Theorem 23.4] \label{thm:subdiff}
  Suppose $g:\Rd \to \mathbb{R}$ is convex and finite everywhere. Then $\partial g(x)$ is a non-empty compact subset of $\Rd$ for all $x\in \Rd$.
\end{theorem}

Since we showed that $f^*$ is convex and finite everywhere, the above theorem, together with Corollary \ref{cor:best}, implies:
\begin{corollary} \label{cor:best-bounded}
Under the assumption of Theorem \ref{thm:convex}, the set $\argmax\limits_{\hat{x}\in \Rd} u(\hat{x},x,\beta)$ is non-empty and bounded for all $\beta\in \Rd$.
\end{corollary}

The boundedness of the set $\argmax\limits_{\hat{x}\in \Rd} u(\hat{x},x,\beta)$ is not directly used in the proof of \Cref{thm:convex}, but it excludes the unrealistic situation in which the strategic agent's best response can be arbitrarily far from the origin.


Step 3 consists of the following claim:
\begin{claim}\label{claim:homo}
	
	If $f:\Rd\to\mathbb{R}$ is convex and homogeneous of degree $r>1$, then $f^*:\Rd\to \mathbb{R}$ is convex and homogeneous of degree $s>1$, where $\frac{1}{r}+\frac{1}{s}=1$.

\end{claim}
Throughout the rest of the paper, $(r,s)$ (and also $(p,q)$ which we will use when discussing norms) will refer to such a pair of numbers unless stated otherwise. For the proof of \Cref{claim:homo}, see \Cref{pf:{claim:homo}}.



Now we move on to step 4. First we show a slightly more general version of Euler's theorem on homogeneous functions.

\begin{theorem}[Euler]
\label{thm:euler}
	Assume $g$ is convex and homogeneous of degree $k>0$, then for any $v\in \partial g(x)$,
	\[
	\langle v, x \rangle = k g(x) .
	\]
\end{theorem}

Note that: 1. This theorem is slightly more general than the standard statement of Euler's theorem in that it can deal with non-differentiable functions; 2. The value of the inner product is independent of the choice of the subgradient, which is interesting. 

\begin{proof}
    Fix any $x$ and consider the function $\bar{g}: [0,\infty) \to \mathbb{R}$ defined by $\bar{g}(\alpha) = g(\alpha x)$.
    Convexity of $g$ implies that $\bar{g}$ is a convex function of $\alpha$:
    \begin{align*}
        \bar{g}(\delta \alpha + (1-\delta)\alpha')
        &= g(\delta (\alpha x) + (1-\delta) (\alpha' x))  \\
        &\leq \delta g(\alpha x) + (1-\delta)g(\alpha' x)  \\
        &= \delta \bar{g}(\alpha) + (1-\delta) \bar{g}(\alpha') .
    \end{align*}
    We consider the set of subgradients of $\bar{g}(\alpha)$ (which is well-defined as $\bar{g}$ is convex).
    By a chain rule for the subgradient (see Theorem 23.9 of \cite{rockafellar2015convex}), they are
    \begin{equation} \label{eqn:subdiff-1}
      \partial \bar{g}(\alpha) = \left\{ \langle x, v \rangle \mid v \in \partial g(\alpha x) \right\}.
    \end{equation}
    On the other hand, by homogeneity, we have $\bar{g}(\alpha) = \alpha^k g(x)$, and this is differentiable with respect to $\alpha$ with $\frac{d\bar{g}}{d\alpha}(\alpha) = k\alpha^{k-1}g(x)$, and in particular $\frac{d\bar{g}}{d\alpha}(1) = kg(x)$.
    So we must have that the set of subgradients of $\bar{g}$ at $\alpha=1$ equals this singleton set, i.e. $\partial \bar{g}(1) = \{kg(x)\}$.
    Combining with (\ref{eqn:subdiff-1}), we get that for all $v \in \partial g(x)$, $\langle x,v \rangle = kg(x)$.

%
\end{proof}

\djs{I think the proof of this can stay in the text, since readers may wonder why breaking ties arbitrarily doesn't matter.}

Combining Claim \ref{claim:homo} and Euler's theorem, if $f$ is convex and homogeneous of degree $r>1$, we have that
$$\xbb = \innprod{x}{\beta}+s f^*(\beta)$$
for any $\hat{x}(\beta)\in \argmax\limits_{\hat{x}\in \Rd} u(\hat{x},x,\beta)$. This shows that $\xbb$ is a well-defined convex function of $\beta$.

\zs{FIXED:Probably want something here to wrap up this proof}

\subsection{Examples}\label{sec:example}

In this section we provide a large class of $f$ that satisfy the assumptions in Theorem \ref{thm:convex}. We will also see that for this class of functions $f$, we can derive simple closed form expressions for $f^*$, which will help us determine the convergence rate for the optimization procedure we develop in Section \ref{sec:opt}. We will also see some examples that violate our assumptions, which demonstrate the necessity of those assumptions: without them, agent best responses are not necessarily well defined.

\begin{claim}\label{claim:f}
	
	For any abstract norm $\|\cdot\|$ on $\Rd$ and any $r>1$
	$$f(x) = \frac{1}{r} \|x\|^r$$ satisfies the conditions of Theorem \ref{thm:convex}. Namely,
	\begin{itemize}
		\item $f(x)>0$ for all $x\neq0$;
		\item $f$ is convex over $\Rd$;
		\item $f$ is positive homogeneous of degree $r>1$.
	\end{itemize}
	Furthermore, let $\|\cdot\|_*$ be the dual norm of $\|\cdot\|$. $f^*$ has the following form:
	\[
	f^*(\beta) = \frac{1}{s} \|\beta\|_*^s.
	\]
	Here $1/r+1/s=1.$ \ar{How come we aren't using the $p$, $q$ notation we already defined for such pairs?}\djs{FIXED}
\end{claim}
The conditions can be easily checked. See \Cref{pf:{claim:f}} for the derivation of $f^*$.


\begin{corollary} \label{cor:norms-convex}
	For any norm $\|\cdot\|$ on $\Rd$, any $r>1$ and any invertible linear transform $A:\Rd\to\Rd$,
	$$f(x) = \frac{1}{r} \|Ax\|^r$$ satisfies all the conditions in Theorem \ref{thm:convex}. $f^*$ has the following form:
	\[
		f^*(\beta) = \frac{1}{s} \|(A^T)^{-1}\beta\|_*^s.
	\]
\end{corollary}
\begin{proof}
Since $\|A(\cdot)\|$ is also a norm, the conditions follow from Claim \ref{claim:f}. The expression for $f^*$ is a result from the following identity:
$$\innprod{x}{\beta} - g(Ax)=\innprod{Ax}{(A^T)^{-1}\beta}- g(Ax).$$
\end{proof}
In particular, if $\|\cdot\|$ is the Euclidean 2-norm and $r=s=2$, $f(x) = \frac{1}{2} \|Ax\|^2$ can be viewed as a form of \emph{Mahalanobis distance}, commonly studied in the metric learning literature~\citep{metriclearning}.

Combining Corollary \ref{cor:norms-convex}, Theorem \ref{thm:convex}, and Theorem \ref{thm:sufficient-convex}, we get the following result.
This general approach can naturally extend to other choices of distance functions as well via the same outline.
\begin{corollary} \label{cor:norms-convex-losses}
  For any norm $\|\cdot\|$ on $\mathbb{R}^d$, any invertible linear transformation $A: \mathbb{R}^d \to \mathbb{R}^d$, and any $r > 1$, if an agent's utility is of the form
    \[ u(\hat{x},x,\beta) = \langle \hat{x},\beta \rangle - \frac{1}{r}\|Ax - A\hat{x}\|^r , \]
  then both the logistic loss and and hinge loss are convex functions of $\beta$, assuming the agent best-responds to $\beta$.
\end{corollary}

\paragraph{Norms and nonexamples.}
\djs{Added discussion on degeneracy}
We give two non-examples, which illustrate the following points:
\begin{enumerate}
	\item There exist $f$ which violate the condition that $f$ is nonvanishing except at the origin, for which strategic agents might not have finite best responses, and;
	\item When we take the degree $r=1$ (violating the $r > 1$ condition), again, strategic agents might not have finite best responses.
\end{enumerate}
The first example shows that the linear transform $A$ we used in \Cref{cor:norms-convex-losses} must be full rank. The second example shows that $f(x) = \|x\|$  where $\|\cdot\|$ is a norm (not raised to any power $r > 1$) also fails to constrain the agent's best responses to be finite

To illustrate the first point, consider a utility function of the form used in \Cref{cor:norms-convex-losses} but with matrix $A$ having non-trivial kernel. Assume $v \in \ker A,v\neq0$. Note that $f(v)=0$ so it violates our non-degeneracy condition. When $\hat{x} = x + \lambda v,\forall \lambda\in\mathbb{R}$, the deviation incurs no cost since $v$ is in the kernel of $A$. Whenever $\beta$ is not orthogonal to $v$, therefore, the agents have no finite best response: they can increase their utility to $+\infty$ by letting $\lambda$ go to $+\infty$ or $-\infty$. Hence, for this choice of $f$, best responses for the agents are not even well defined.

It is also natural to ask what happens when $f$ is a norm (rather than a norm raised to some power $r$ greater than one).
After all, any norm is positive homogeneous of degree 1.
In this case too, agent best responses are not necessarily well defined.
The issue is that norms do not exhibit \emph{diminishing returns} to the agent as he moves $\hat{x}$ increasingly far from $x$.
Thus, the agent's best response would always be either to not manipulate at all (resulting in truthfully setting $\hat{x} = x$), or else to manipulate as much as possible, with arbitrarily high reward achievable.

\begin{claim} \label{claim:norms-no-good}
  Suppose the agent's distance function is $d(\hat{x},x)=\|\hat{x}-x\|$ for some norm $\|\cdot\|$, and let $\|\cdot\|_*$ be the dual norm.
  Then his best response is $\hat{x} = x$ for $\|\beta\|_* \leq 1$; and if $\|\beta\|_* > 1$, his best response is undefined with arbitrarily high utility achievable.
\end{claim}
\begin{proof}
  Let $f$ be some norm, i.e. $f(z) = \|z\|$.
  The agent wishes to solve
  \begin{align*}
    \sup_{\hat{x}} ~ (\langle \hat{x}, \beta \rangle &- f(\hat{x}-x))
      = \\
      &=\langle x,\beta \rangle + \sup_{\hat{x}} ~ (\langle \hat{x} - x, \beta \rangle - f(\hat{x}-x))  \\
      &= \langle x,\beta \rangle + \sup_{x'} \langle x',\beta \rangle - f(x')  \\
      &= \langle x,\beta \rangle + f^*(\beta)
  \end{align*}
  where $f^*$, the conjugate of $f$, is the $(0-\infty)$ indicator function of the unit ball in the dual norm~\citep{rockafellar2015convex}, namely
    \[
    f^*(\beta) = \begin{cases}
      0 & ||\beta||_* \leqslant1 \\
      \infty  & \text{otherwise.}
      \end{cases}
    \]
  This implies the theorem.
\end{proof}

%% file: opt.tex
\section{Regret Minimization}\label{sec:opt}
In this section, we will present an online convex optimization method
that has sub-linear Stackelberg regret. Our algorithm is a slight
variant of the bandit algorithm in \cite{flaxman2005online}, which
only uses zeroth-order feedback. That is, it only observes information of the form (input, function value).
In our setting, the learner always has access to zeroth-order information, as she observes the loss function at a particular choice of parameters $\beta_t$.
However, when a non-strategic agent arrives, the learner can also deduce \emph{first-order} information, that
is, the subgradient of the loss evaluated at $\beta_t$.
Our algorithm takes advantage of this mixture of feedback
and performs a full subgradient descent update for every non-strategic
round. As a result, we obtain a regret rate that depends on the
fraction of strategic agents that show up.


\subsection{Convex Optimization with Mixture Feedback}

In this section, we abstract away from the particular forms of the loss function until they become relevant.
At a high level, our algorithm is running (stochastic) subgradient descent to minimize a sequence of convex loss functions $c_t(\beta)$ while the learner proposes a sequence of parameters $\beta_t$.
The key challenge is to obtain subgradient feedback based on the agent's data
$(\hat x, y)$.  In each non-strategic round, we can explicitly compute
a subgradient in $\partial c_t(\beta_t)$. In other words, we have first-order information.
In each strategic round,
however, since $\hat x_t$ is a function of $\beta_t$, we do not have
direct access to the subgradients of $c_t$, and only see the ``zeroth-order'' information $c_t(\beta_t)$.

To estimate a subgradient with this information, we utilize the approach of \citet{flaxman2005online}
and minimize
a ``smoothed'' version of the classification loss. In particular, for
any $c_t$ and parameter $\delta$, the $\delta$-smoothed version of
$c_t$ is defined as:
\[
  \tilde{c_t} (\beta) = \mathbb{E}[c_t(\beta+\delta S)]
\]
where $S$ is a random vector drawn from the uniform distribution over
the unit sphere in $\mathbb{R}^d$. Notice that $\tilde c_t$ is both convex and differentiable.  \citet{flaxman2005online} provides
a method to obtain an unbiased estimate for the gradient of the
smoothed loss $\tilde c$. This is done by perturbing $\beta_t$ to get
a noisy $\beta_t^+$ and then evaluating the classification loss at
$\beta_t^+$. More formally:

\begin{lemma}[\cite{flaxman2005online}]\label{lem:grad}
  Let $\beta_t^+ = \beta_t + \delta\cdot S$ and
  $G_t = \frac{d}{\delta}\, c_t(\beta_t^+)\, S$, where $S$ is a random
  vector drawn from the uniform distribution over the unit sphere in
  $\mathbb{R}^d$. Then
  $\mathbb{E}[G_t] = \nabla \tilde{c_t}(\beta_t)$.
\end{lemma}

Note that the perturbed point $\beta_t^+$ may fall outside of the
feasible set $K$. To prevent that, we will optimize over a strict
subset $K_\delta = \{(1 - \delta)\beta \mid \beta \in
K\}$. Then we can guarantee that for any $\beta\in K_\delta$, the
$\delta$-ball centered at $\beta$ is fully contained in $K$.\footnote{
    Any $\beta_t^+ = (1-\delta)\beta + \delta v$ for some $\beta \in K$ and some unit vector $v$.
    We assume $K$ contains the unit ball, in particular, both $v$ and $\beta$ are in the convex set $K$, and $\beta_t^+$ is a convex combination of them so it also lies in $K$.}
    } 

Let $\theta = \frac{\#\{t:y_t=-1\}}{n}$ be the
fraction of rounds in which the learner interacts with a strategic agent. For simplicity, we assume for now that $\theta$ is known to the algorithm in advance, and use it to set the parameters $\eta$ and $\delta$ in our algorithm. However, as we discuss at the end of this section, any upper bound $\hat{\theta}$ on $\theta$ suffices, and can be used in its place, at the cost of only an additional additive constant to the final regret bound.
We are now ready to present our main algorithm in
\Cref{alg:mix}.

\begin{algorithm}[h]
  \caption{Convex optimization with mixture feedback}
 \label{alg:mix}
  \begin{algorithmic}
    \STATE{\textbf{Input:} smoothing parameter $\delta = \theta^{1/4}\cdot\sqrt{\frac{dMR}{L(R+3)}}\cdot n^{-1/4}$, subset to
      optimize over $K_\delta \subset K$, step size $\eta = R/\sqrt{n(\theta\cdot\frac{d^2M^2}{\delta^2}+(1-\theta)L^2)}$ \sw{DONE:todo:
        define}}

    \STATE{\textbf{For} each round $t$:}
    \INDSTATE Let $S_t$ be drawn uniformly from the unit sphere;
    \INDSTATE Let $\beta_t^+ = \beta_t + \delta\cdot S_t$;
    \INDSTATE Query at
    $\beta_t^+$, observe corresponding $\hat{x}_t, y_t$, and suffer
    loss $c_t(\beta_t^+)$;
    \INDSTATE Compute stochastic subgradient:
    \INDSTATE{\textbf{If} $y_t = 1$, then set $g_t$ as a subgradient
      in $\partial c_t(\beta_t)$; \textbf{otherwise} set
      $g_t = \frac{d}{\delta}\cdot c_t(\beta_t^+)\cdot S_t$ .}
    \INDSTATE \textbf{Update:}
    $\beta_{t+1} = \Pi_{K_\delta}(\beta_t - \eta g_t)$, where $\Pi_{K_\delta}$ is the Euclidean projection map onto the set $K_\delta$.
    \end{algorithmic}
  \end{algorithm}

Now we will proceed to bound the regret rate of our algorithm.  To
facilitate our analysis, we will make the following well-behaved
assumption on the loss function $c_t$ for all rounds $t$:
\begin{itemize}
\item $|c_t(\beta)|\leqslant M$ for any $\beta\in K$, and
\item $c_t$ is $L$-Lipschitz (with respect to the $\ell_2$ norm) over
  the set $K$. 
\end{itemize}

In \Cref{sec:const}, we will work out the parameters $M$ and $L$ for both
logistic loss and hinge loss, along with distance function of the form
$d_t(x,y) = f(A_t(x-y))$ for the strategic agents.


The following lemma allows us to bound the regret separately in
non-strategic and strategic rounds, and also bounds the additional
regret due to the $\delta$-smoothing of our
losses.

\begin{lemma}\label{lem:tilde}
  Let $T_{1}:=\{t:y_t=1,t=1,2,\ldots,n\}$ be the non-strategic rounds
  and $T_{-1}:=\{t:y_t=-1,t=1,2,\ldots,n\}$ be the strategic
  rounds. Then for any $\beta\in K$,
\[
  \mathbb{E}\left[\sum_{t=1}^n c_t(\beta_t^+)-c_t(\beta)
  \right]\leqslant\mathbb{E}\left[\sum_{t\in T_{-1}}
  \tilde{c_t}(\beta_t)-\tilde{c_t}(\beta) \right]
  +
  \mathbb{E}\left[\sum_{t\in T_{1}} c_t(\beta_t)-c_t(\beta) \right]
  + 3nL\delta,
\]
where the expectation is taken over the randomness of the algorithm.
\end{lemma}

\begin{proof}
  Fix any $\beta\in K_\delta$.  Since the vector $S$ is drawn from the
  unit sphere,
  $|c_t(\beta+\delta S) - c_t(\beta)|\leqslant L\|\delta S\|= L
  \delta$. Hence by taking expectation over $S$,
\begin{equation}\label{eqn:Ld}
|\tilde{c_t} (\beta) - c_t(\beta)|\leqslant L \delta.
\end{equation}

Recall that $\beta_t^+ = \beta_t + \delta\cdot S$, by Lipschitz
property again, $|c_t(\beta_t^+) - c_t(\beta_t)| \leqslant L\delta$.
For any $t\in T_{-1}$, we have
\begin{align}\label{eqn:2Ld}
\begin{split}
|c_t(\beta_t^+) - \tilde{c_t} (\beta_t)| &\leqslant |c_t(\beta_t^+) - c_t(\beta_t)| + | c_t(\beta_t) - \tilde{c_t} (\beta_t)| \\
&\leqslant L \delta +L \delta = 2L \delta.
\end{split}
\end{align}
Now, we want to bound the difference between the two expressions,
$c_t(\beta_t^+)-c_t(\beta)$ and
$\tilde{c_t}(\beta_t)-\tilde{c_t}(\beta)$. By \Cref{eqn:2Ld}, we have
$|c_t(\beta_t^+) - \tilde{c_t}(\beta_t)| \leq 2L\delta$, and by \Cref{eqn:Ld}, we have
$|c_t(\beta) - \tilde{c_t}(\beta)| \leq L\delta$. It follows that
\[
c_t(\beta_t^+)-c_t(\beta)\leqslant \tilde{c_t}(\beta_t)-\tilde{c_t}(\beta) + 3L\delta.
\]
Similarly, for any $t\in T_1$, we can bound
\[
c_t(\beta_t^+)-c_t(\beta)\leqslant c_t(\beta_t)-c_t(\beta) + L\delta.
\]
Summing over all rounds $t$, we recover the stated bound.
\end{proof}

By bounding the regret separately in non-strategic and strategic
rounds, we can then bound the Stackelberg regret of \Cref{alg:mix}.
\begin{theorem}\label{thm:regret}
  For any sequence of agents $A = (a_1, \ldots, a_n)$, \Cref{alg:mix}
  will output a sequence of classifiers
  $B= (\beta_1, \ldots, \beta_n)$ such that the expected Stackelberg regret  satisfies
\[
  \mathbb{E}\left[ \mathcal{R}_S(A,B) \right]
  \leqslant \frac{\eta}{2}[n\theta\cdot\frac{d^2M^2}{\delta^2} +
  n(1-\theta)L^2]
  + \frac{1}{2\eta}R^2 + 3nL\delta + n L R\delta.
\]
\end{theorem}

\begin{proof}
  Our proof essentially follows from the standard analysis for online
  subgradient descent.  We will first bound
$$\mathbb{E}\left[\sum_{t\in T_{-1}} \tilde{c_t}(\beta_t)-\tilde{c_t}(\beta^*)\right] + \mathbb{E}\left[\sum_{t\in T_{1}} c_t(\beta_t)-c_t(\beta^*)\right]$$
which then allows us to apply Lemma \ref{lem:tilde}. By convexity, we get for any $\beta \in K$,
\[
\tilde{c_t}(\beta_t) - \tilde{c_t}(\beta)\leqslant \langle \nabla \tilde{c_t}(\beta_t), \beta_t - \beta \rangle
\]
and for any subgradient $G_t(\beta)\in\partial c_t(\beta_t)$,
\[
{c_t}(\beta_t) - {c_t}(\beta)\leqslant \langle G_t(\beta), \beta_t - \beta \rangle.
\]

\begin{align*}
\tilde{c_t}(\beta_t^+) - \tilde{c_t}(\beta) &=
[\tilde{c_t}(\beta_t^+) - \tilde{c_t}(\beta_t)] + [\tilde{c_t}(\beta_t) - \tilde{c_t}(\beta)]\\
							   &\leqslant  L_0 \delta + \tilde{c_t}(\beta_t) - \tilde{c_t}(\beta)
\end{align*}
Recall that for any $t\in T_{1}$, $g_t \in \partial {c_t}(\beta_t)$,
and by \Cref{lem:grad}, we know for $t\in T_{-1}$, we can write
$g_t = \nabla \tilde{c_t}(\beta_t) - \xi_t$ where $\xi_t$ is a random
vector with zero mean. Let
$\beta' \in \argmin_{\beta\in K_\delta} \sum_{t=1}^n c_t(\beta)$, then
\[
  \sum_{t\in T_{-1}}  [\tilde{c_t}(\beta_t) -\tilde{c_t}(\beta') ]
 + \sum_{t\in T_{1}} \left[ c_t(\beta_t)-c_t(\beta') \right]
 \leqslant \sum_{t=1}^n  \langle g_t, \beta_t - \beta' \rangle + \sum_{t\in T_{-1}} \langle \xi_t, \beta_t - \beta' \rangle.
\]
Taking the expectation, we have
\begin{equation}\label{eqn:twoSum}
\mathbb{E}[\sum_{t\in T_{-1}} \tilde{c_t}(\beta_t)-\tilde{c_t}(\beta')] + \mathbb{E}[\sum_{t\in T_{1}} c_t(\beta_t)-c_t(\beta')]\\
\leqslant \mathbb{E}\sum_{t=1}^n \langle g_t, \beta_t - \beta' \rangle.
\end{equation}
Now we can employ the standard cosine law trick (see e.g. \cite{bubeck2015convex}, \cite{ben2001lectures}) and get:
$$\langle g_t, \beta_t - \beta' \rangle \leqslant \frac{\eta}{2}\|g_t\|^2+ \frac{1}{2\eta} (\|\beta' - \beta_t\|^2 -\|\beta' - \beta_{t+1}\|^2)$$
Summing them up, telescoping and leaving out the negative term yields:
\begin{equation}\label{eqn:g_t}
\sum_{t=1}^n \langle g_t, \beta_t - \beta' \rangle \leqslant  \frac{\eta}{2} \sum_{t=1}^n \|g_t\|^2 + \frac{1}{2\eta}\|\beta' - \beta_1\|^2
\end{equation}
Since we have
\[
\|g_t\| \leqslant \left\{
			\begin{array}{ll}
			L, 		& y_t=0, \\
			\frac{d}{\delta}\cdot M, & y_t=1,
			\end{array}
			\right.
\]
Combining (\ref{eqn:twoSum}) and (\ref{eqn:g_t}) and using the above bound for $g_t$, we get
\begin{equation}\label{fun}
\mathbb{E}[\sum_{t\in T_{-1}} \tilde{c_t}(\beta_t)-\tilde{c_t}(\beta')] + \mathbb{E}[\sum_{t\in T_{1}} c_t(\beta_t)-c_t(\beta')]
 \leqslant \frac{\eta}{2} [n\theta\cdot\frac{d^2M^2}{\delta^2} + n(1-\theta)L^2] + \frac{1}{2\eta}R^2.
\end{equation}
Finally, we need bound the additional error we incur for restricting
to the subset $K_\delta$. Since for any point $\beta\in K$, the closet
point to $\beta$ in $K_\delta$ has $\ell_2$ distance no more than
$\delta R$, by the Lipschitz property of each $c_t$, we know
\begin{equation}\label{add}
  \sum_{t=1}^n c_t(\beta') - \min_{\beta\in K} \sum_{t=1}^n c_t(\beta)
  \leq n L R\delta
\end{equation}
Our regret bound then follows from \Cref{fun} and \Cref{add}.
\end{proof}


In
\Cref{cor:regret}, we provide a regret rate for three different regimes
of $p$ (hiding dependence on $M$, $L$, and $R$).

\begin{corollary}\label{cor:regret}
For different ranges of $\theta$, we can bound the expected Stackelberg regret by
\[
\mathbb{E}\left[\mathcal{R}_S(A, B)\right] \leqslant \left\{
			\begin{array}{ll}
			O(\sqrt{n}), 		& \theta=0, \\
			O(\sqrt{d}n^{3/4}), 		& \theta=\Omega(1), \\
			O(\sqrt{d}n^{1-\frac{1+\gamma}{4}})				& \theta=O(n^{-\gamma})
			\end{array}
			\right.
\]
\end{corollary}
\begin{proof}
Plugging the algorithm parameter $\eta=R/\sqrt{n(\theta\cdot\frac{d^2M^2}{\delta^2}+(1-\theta)L^2)}$ into the regret bound in \Cref{thm:regret},
\[
\mathcal{R}_S \leqslant \sqrt{\theta\cdot\frac{d^2M^2}{\delta^2}+(1-\theta)L^2}\cdot\sqrt{n}\cdot R + n\delta L(R+3).
\]
Using $\sqrt{a+b}\leqslant \sqrt{a} + \sqrt{b}$, we can further simplify the bound:
\[
\mathcal{R}_S \leqslant \sqrt{n\theta}\cdot\frac{dMR}{\delta}+LR\sqrt{n(1-\theta)} + n\delta L(R+3).
\]
Plugging in $\delta = \theta^{1/4}\cdot\sqrt{\frac{dMR}{L(R+3)}}\cdot n^{-1/4}$ yields:
\[
\mathcal{R}_S \leqslant LR\sqrt{n(1-\theta)} + n^{3/4}\sqrt{dMLR(R+3)}\theta^{1/4}.
\]
The rest of the result follows by setting $\theta=0, \Omega(1)$ or $O(n^{-\gamma})$.
\end{proof}

  \paragraph{Relaxing $\theta$ to $\hat{\theta}$}{
    Let $\hat{\theta}$ be an upper bound on the true fraction of strategic agents in rounds $1,2,\ldots,n$, i.e. $\theta\leqslant\hat{\theta}$. We briefly explain why $\hat{\theta}$ can be used in place of $\theta$ in setting the parameters of \Cref{alg:mix}, and \Cref{cor:regret} still holds (with $\hat{\theta}$ replacing $\theta$ in the bound.)
    First, consider the case in which $\frac{dM}{\delta} \geqslant L$. In this case, the right hand side of the bound in \Cref{thm:regret} is increasing in $\theta$, so we can replace $\theta$ with $\hat{\theta}$ and still get a correct regret bound:
    \[
    \mathbb{E}\left[ \mathcal{R}_S(A,B) \right]
  \leqslant \frac{\eta}{2}[n\hat{\theta}\cdot\frac{d^2M^2}{\delta^2} +
  n(1-\hat{\theta})L^2]
  + \frac{1}{2\eta}R^2 + 3nL\delta + n L R\delta.
    \]
    Thus, we can carry out the calculations from the proof of \Cref{cor:regret}, again with $\hat{\theta}$ replacing $\theta$. Finally, it is not hard to show that when $n\geqslant(\frac{L}{dM})^2$, we are in the case in which $\frac{dM}{\delta} \geqslant L$. Since we assume $c_t(\beta)\leqslant M$, each round incurs regret at most $2M$. Therefore in the remaining case in which $n<(\frac{L}{dM})^2$, we can upper bound the regret by $2M \cdot n \leqslant 2M \cdot (\frac{L}{dM})^2$. Hence, we obtain the regret bound
    \[
    \mathcal{R}_S \leqslant \max\bigg\{LR\sqrt{n(1-\hat{\theta})} + n^{3/4}\sqrt{dMLR(R+3)}\hat{\theta}^{1/4}, \frac{2L^2}{d^2M} \bigg\}.
    \]
    The new (second) term in the maximum is independent of $n$, and so does not change the asymptotic bound stated in \Cref{cor:regret}.
}

\paragraph{Applying other tools for bandit convex optimization}{ We can in fact
  apply other tools for bandit convex optimization to minimize
  regret. For example, we can use the algorithm in
  \cite{bubeck2016kernel}, which will achieve a Stackelberg regret rate of
  $\tilde O(d^{9.5} \sqrt{n})$ even when all the agents are strategic (that is $p =1$).
  Our algorithm, which is a variant of \cite{flaxman2005online}, has a regret rate with
  a worse dependence on the time horizon, but has a milder dependence on the dimension. Furthermore, it also
  permits a natural interpolation between the fully strategic and fully non-strategic settings (as shown in \Cref{cor:regret}).
}

\paragraph{Deriving the Constants}{
In Appendix \ref{sec:const}, we first illustrate how to determine the constants $L$ and $M$ that appear in our regret bound for general classes of cost functions $d(x,y)$, namely any abstract norm raised to some power greater than 1:
\[
d(x,y) = \frac{1}{r}\|x-y\|^r.
\]
The constant $C$ in the inequality $\|x\|_2\leqslant C\|x\|$ can be used to determine $L$ and $M$. See \Cref{thm:norm} for details.

We also give explicit expressions for $L$ and $M$ when the abstract norm is some $p$-norm with $p\geqslant 1$. See \Cref{cor:pnorm_const}.

All assumptions and the resulting fully expanded regret bound are described in \Cref{thm:final_rate}. We state here a noteworthy special case that has an especially mild dependence on the dimension $d$:

When $1\leqslant p\leqslant 2$, and for any power $r > 1$, \Cref{alg:mix} has regret bound:
\[
\mathbb{E}\left[ \mathcal{R}_S(A,B) \right] \leqslant \left\{
      \begin{array}{ll}
      O(\sqrt{n}),    & \theta=0, \\
      O(\sqrt{d}n^{3/4}),     & \theta=\Omega(1), \\
      O(\sqrt{d}n^{1-\frac{1+\gamma}{4}}),        & \theta=O(n^{-\gamma}).
      \end{array}
      \right.
\]
i.e. for $p$ norms when $1\leqslant p\leqslant 2$, regardless of the power $r > 1$, we obtain a fixed (and mild) dependence on $d$. 
When $p>2$, the situation changes: our dependence on $d$ grows with $p$, and shrinks with $r$.

}


%% file: future.tex
\section{Discussion and Open Questions} \label{sec:future}
Our work suggests a number of interesting directions. Broadly speaking, these can be grouped into two thrusts: \emph{broadening} the class of problems to which our approach can be applied, and \emph{weakening} the assumptions that need to be made about the agents. Here we list a concrete example of each type of question. 

The broad approach we take in this paper is to identify classes of utility functions for the strategic agents, that when paired with natural cost functions for the learner, lead to a convex objective, even in the strategic setting. However, this precludes several natural utility functions for the agents: most notably, utility functions that are defined in terms of 0/1 classification loss. What can be done in the face of non-convexity? When agent cost functions $d$ are known and \emph{separable}, \cite{HMPW16} show that even for 0/1 loss, Stackelberg equilibria have structure that allows for their efficient computation, despite non-convexity of the learner's objective. Can similar structure be taken advantage of in the setting we consider, when agent cost functions are not known to the learner? 

Next, the overarching motivation of this paper is to weaken the assumptions on the knowledge of the learner necessary to solve the strategic classification problem. We should also aim to weaken other assumptions. For example, we assume in this paper that strategic agents play exact best responses --- i.e. the always exactly optimize their utility function. What if strategic agents are not perfect -- i.e. they are only guaranteed to play approximate best responses (or, perhaps only \emph{usually} guaranteed to). Is it possible to give optimization procedures with guarantees that are robust to strategic agent imperfection?

%% file: const.tex
\section{Concrete Regret Bounds for Specific Examples}\label{sec:const}
We bounded the Stackelberg regret of our algorithm in terms of problem-dependent constants $L$ and $M$. In this section, we go through several examples, in which we work out explicit bounds on $L$ and $M$, first for abstract norms, then for the class of $L^p$ norms.
\paragraph{Note:}{In this section, the norm notation without any subscript, $\|\cdot\|$, denotes an abstract norm in $\Rd$. $\|\cdot\|_*$ denotes its dual norm. $\|\cdot\|_p$ denotes the Euclidean $p$-norm, i.e. $\|x\|_p^p = \sum_{i=1}^d |x_i|^p$ when $1\leqslant p <+\infty$ and $\|x\|_\infty = \max_i |x_i|$.}
\vspace{2mm}\\
Recall from the previous section that when the distance function $d(x,y)$
in the strategic agent's utility function has the form
$$
d(x,y) = \frac{1}{r}\|x-y\|^r
$$
we showed that
$$
\langle \beta, \hat{x}(\beta) \rangle = \langle \beta, x \rangle + \|\beta\|_*^s.
$$
Again $1/r+1/s=1, r,s>1$.\\
We make the following assumptions:
\begin{itemize}
	\item $\|x\|_2 \leqslant R_1$. That is, the ``real'' feature vectors lie in a ball of radius $R_1$;
	\item $\|\beta\|_2 \leqslant R_2$. That is, our target parameters lie in a ball of radius $R_2$.
\end{itemize}
With the above assumptions, we have the following theorem:
\begin{theorem}\label{thm:norm}
For logistic loss, $h(z) = \log(1+e^{-z})$, and hinge loss, $h(z) = (1-z)_+$, and for any norm $\|\cdot\|$ such that  $\|x\|_2 \leqslant C \|x\|$, the loss function $c(x,y,\beta) = h(y\langle \beta, \hat{x}(\beta) \rangle)$ has the following properties:
\begin{enumerate}
	\item $|c(x,y,\beta)|\leqslant 1+R_1R_2+ C^s R_2^s.$ 
	\item $c(x,y,\beta)$ is $L$-Lipschitz in $\beta$, where $L$ can be taken to be
	\[
	R_1 + sC^sR_2^{s-1}.
	\]
\end{enumerate}
\end{theorem}
\begin{remark}
For any norm, there is some such constant $C$ as required in the theorem, since any two norms on $\Rd$ are equivalent. Note that $C$ may depend on the  dimension $d$.
\end{remark}
\begin{proof}
Since the 2-norm is self-dual, the following lemma holds
\begin{lemma}\label{lem:dual} The norm pair $(\|\cdot\|,\|\cdot\|_*)$ and constant $C>0$ satisfy
$$\|x\|_2 \leqslant C \|x\|, \forall x\in\Rd$$ if and only if they satisfy $$\|\beta\|_* \leqslant C \|\beta\|_2$$.
\end{lemma}
\begin{proof}
We only show the ``only if'' direction. The other follows exactly the same reasoning.\\
Assume that $\|x\|_2 \leqslant C \|x\|, \forall x\in\Rd$. Then $\{x:\|x\|=1\}$ is a subset of $\{x:\|x\|_2\leqslant C\}$. Hence for any $\beta\in\Rd$,
$$
\|\beta\|_* = \sup_{x:\|x\|=1} \innprod{\beta}{x} \leqslant \sup_{x:\|x\|_2\leqslant C}\innprod{\beta}{x} = C\cdot \sup_{x:\|x\|_2\leqslant 1}\innprod{\beta}{x} = C\|\beta\|_2.
$$
This proves the ``only if'' direction.
\end{proof}
Both $h(t) = \log(1+e^{-t})$ and $h(t) = (1-t)_+$ satisfy
\[0\leqslant h(t) \leqslant |t|+1, |h'(t)|\leqslant 1\]
for any $t\in\mathbb{R}$.

As a consequence,
\begin{align*}|c(x,y,&\beta)| = |h(y\langle \beta, \hat{x}(\beta) \rangle)|\leqslant\\
 &\leqslant 1 + |y\langle \beta, \hat{x}(\beta) \rangle)| = 1 + |\langle \beta, \hat{x}(\beta) \rangle)|.
\end{align*}
So in order to bound $|c(x,y,\beta)|$, it suffices to bound $|\langle \beta, \hat{x}(\beta) \rangle|$ as follows:
\begin{align*}
|\langle \beta, \hat{x}(\beta) \rangle| &\leqslant |\langle \beta, x \rangle| + \|\beta\|_*^s \\
&\leqslant \|\beta\|_2\cdot\|x\|_2 + (C\|\beta\|_2)^s\\
& \leqslant R_1R_2+ C^s R_2^s.
\end{align*}
This yields the bound on $|c(x,y,\beta)|$ as stated in the theorem. 
Since $|h'(t)|\leqslant 1$, the Lipschitz constant $L$ of $|c(x,y,\beta)|$ can be taken to be the Lipschitz constant for the function $\langle \beta, \hat{x}(\beta) \rangle$. 
Recall that $\langle \beta, \hat{x}(\beta) \rangle = \langle \beta, x \rangle + \|\beta\|_*^s.$
So
\begin{itemize}
	\item The Lipschitz constant of the function $\beta\mapsto \langle \beta, x \rangle$, with respect to the 2-norm, is $\|x\|_2$;
	\item The Lipschitz constant for the function $\|\beta\|_*\mapsto \|\beta\|_*^s$ is $s(\sup \|\beta\|_*)^{s-1}$;
	\item The Lipschitz constant for the function $\beta\mapsto \|\beta\|_*$, with respect to the 2-norm, is $C$, as we have shown in the beginning of the proof that $\|\beta\|_* \leqslant C \|\beta\|_2$.
\end{itemize}
Therefore, $L$ is bounded by
$$
	\|x\|_2 + s(\sup \|\beta\|_*)^{s-1}\cdot C\leqslant R_1 + sC^sR_2^{s-1}.
$$
\end{proof}
\begin{corollary}\label{cor:pnorm_const}
When $\|x\|$ has the form $\|x\| = \|Ax\|_p$ and hence $d(x,y) = \frac{1}{r}\|A(x-y)\|_p^r$  with the additional assumption that $\sigma_d(A) \geqslant \varepsilon$ where $\sigma_d(A)$ denotes the smallest singular value of $A$, $c(x,y,\beta)$ has the following properties:
\begin{enumerate}
	\item $|c(x,y,\beta)|\leqslant 1+R_1R_2+ \varepsilon^{-s}d^{(s/q-s/2)_+} \cdot R_2^s$; %
	\item $c(x,y,\beta)$ is $L$-Lipschitz in $\beta$, where $L$ can be taken to be
	\[
	R_1 + s\cdot \varepsilon^{-s}d^{(s/q-s/2)_+}\cdot R_2^{s-1}.
	\]
Here $1/p+1/q=1, 1\leqslant p,q \leqslant +\infty$. \ar{Should this be $r$ and $s$?}
\end{enumerate}
\end{corollary}
\begin{remark}
The condition on the smallest singular value of $A$ amounts to saying that $A$ is $\varepsilon$-far from being singular. Mathematically, the operator norm $\|(A^T)^{-1}\|_2 = \|A^{-1}\|_2 = \sigma_d(A)^{-1} \leqslant \varepsilon^{-1}$. \djs{CONSISTENCY} Recall that when $A$ is singular a strategic agent can move $x$ in some beneficial direction for free, resulting in her best response typically being undefined.
\end{remark}
\begin{proof}
We need the following lemma relating a general $r$-norm and 2-norm.
\begin{lemma}\label{lem:pnorm}
Let $x\in\mathbb{R}^d$. When $q\geqslant 2$, \ar{Should this be $r$?}
\[\|x\|_q\leqslant\|x\|_2.\]
When $0<q\leqslant2$,
\[\|x\|_q\leqslant d^{1/q-1/2}\|x\|_2.\] Both of these bounds are tight.
A simplified statement is:\\
For any $q>0$,
\[
\|x\|_q \leqslant d^{(1/q-1/2)_+} \cdot\|x\|_2.
\]
\end{lemma}
By Theorem \ref{thm:norm}, it suffices to determine the constant $C$. When $\|x\|$ has the form $\|x\| = \|Ax\|_p$, its dual norm is $\|\beta\|_* = \|(A^T)^{-1}\beta\|_q$. By Lemma \ref{lem:dual}, it suffices to determine $C$ from expression
\[
\|(A^T)^{-1}\beta\|_q \leqslant C\|\beta\|_2.
\]

Apply Lemma \ref{lem:pnorm} and we have
\begin{align*}
\|(A^T)^{-1}\beta\|_q &\leqslant d^{(1/q-1/2)_+}\|(A^T)^{-1}\beta\|_2 \\
&\leqslant d^{(1/q-1/2)_+}\|(A^T)^{-1}\|_2\cdot\|\beta\|_2 \\
&\leqslant d^{(1/q-1/2)_+}\varepsilon^{-1}\cdot\|\beta\|_2.
\end{align*}
Hence $C$ can be taken to be $\varepsilon^{-1}d^{(1/q-1/2)_+}$. Plugging this into Theorem \ref{thm:norm} yields the claimed result.

\end{proof}
In the end, we present all the assumptions we have made and the resulting fully expanded regret bound.
\begin{theorem}\label{thm:final_rate}
Assume the following for $t=1,2,\ldots,n$:
\begin{itemize}
	\item The strategic agents have utility of the following form:
		$$u_t(x,\beta) = \innprod{\beta}{x} - d_t(x_t,x) = \innprod{\beta}{x} - \frac{1}{r}\|A_t(x_t-x)\|^r_p$$ with $1\leqslant p \leqslant +\infty$ and $1< r < +\infty$;
	\item $\|x_t\|_2 \leqslant R_1$;
	\item The smallest singular value of $A_t$, $\sigma_d(A_t)$ satisfies $\sigma_d(A_t) \geqslant \varepsilon$.
\end{itemize}

Then \Cref{alg:mix} incurs regret
\[
\mathbb{E}\left[ \mathcal{R}_S(A,B) \right] \leqslant \left\{
			\begin{array}{ll}
			O(\sqrt{n}), 		& \theta=0, \\
			O(n^{3/4}d^{(s/q-s/2)_+ + 1/2}), 		& \theta=\Omega(1), \\
			O(n^{\frac{3-\gamma}{4}}d^{(s/q-s/2)_+ + 1/2}),				& \theta=O(n^{-\gamma}).
			\end{array}
			\right.
\]
In particular, when $1\leqslant p\leqslant 2$, regardless of the power $r$, we have the regret bound simplified as:
\[
\mathbb{E}\left[ \mathcal{R}_S(A,B) \right] \leqslant \left\{
			\begin{array}{ll}
			O(\sqrt{n}), 		& \theta=0, \\
			O(\sqrt{d}n^{3/4}), 		& \theta=\Omega(1), \\
			O(\sqrt{d}n^{1-\frac{1+\gamma}{4}}),				& \theta=O(n^{-\gamma}).
			\end{array}
			\right.
\]
\end{theorem}
\begin{remark} In terms of the dependence of $\mathbb{E}\left[ \mathcal{R}_S(A,B) \right]$ on the dimension $d$, we can see from the above theorem:
\begin{itemize}
	\item When $1\leqslant p\leqslant 2$, regardless of the power $r$, we achieve small dependence on $d$;
	\item When $p>2$, the larger $p$ is and the smaller power $r$, the more severe our dependence on $d$ becomes.
	\end{itemize}
\end{remark} 

%% file: appendix.tex
\section{Missing Proofs in \Cref{sec:convex}}

\subsection*{Proof of \Cref{claim:finite}}\label{pf:{claim:finite}}
\begin{claim-non}

Under the assumptions in Theorem \ref{thm:convex}, $f^*(\beta)$ is finite for all $\beta\in\Rd$.
\end{claim-non}

\begin{proof}

We need the following lemma:

\begin{lemma}[\citet{rockafellar2015convex}, Theorem 10.1, Corollary 10.1.1]
A convex function finite on all of $\Rd$ is necessarily continuous.
\end{lemma}

This lemma shows continuity of $f$ under our assumptions.

In particular, letting $\alpha\downarrow0$ in the homogeneity condition $f(\alpha x) = \alpha^pf(x)$ implies that $f(0)=0$.

Let $S^{d-1}$ be the unit sphere in $\Rd$. If we take the supremum to find the convex conjugate $f^*$ with respect to polar coordinates, we have
\begin{align*}
f^*(\beta):=\sup_{x\in \Rd} \innprod{\beta}{x} - f(x) &= \sup_{v\in S^{d-1}}(\sup_{\rho\geqslant0} (\innprod{\beta}{\rho v} - f(\rho v)))\\
&= \sup_{v\in S^{d-1}}(\sup_{\rho r\geqslant0} (\rho \innprod{\beta}{v} - \rho^p f(v)))
\end{align*}
Since $v\in S^{d-1}$, it cannot be $0$. By our first assumption on $f$, $f(v) > 0$. The inner supremum can be computed explicitly:
\[
\sup_{\rho r\geqslant0} (\rho \innprod{\beta}{v} - \rho^pf(v))= \frac{r^{1-s}}{s}\innprod{\beta}{v}^s f(v)^{1-s}.
\]
By the continuity lemma, the above expression is a continuous function of $v$. Taking the supremum over a compact set $S^{d-1}$ can only yield a finite value, hence
\[
f^*(\beta)=\sup_{v\in S^{d-1}}\frac{r^{1-s}}{s}\innprod{\beta}{v}^s f(v)^{1-s} < +\infty.
\]
Therefore, $f^*$ is finite everywhere. \ar{It looks like $p,q$ in this proof should be switched to $r,s$ like elsewhere in the document.}\djs{FIXED}
\end{proof}

\subsection*{Proof of \Cref{claim:homo}}\label{pf:{claim:homo}}
\begin{claim-non}

    If $f:\Rd\to\mathbb{R}$ is convex and homogeneous of degree $r>1$, then $f^*:\Rd\to \mathbb{R}$ is convex and homogeneous of degree $s>1$, where $\frac{1}{r}+\frac{1}{s}=1$.

\end{claim-non}

\begin{proof}
    
    We begin by making the following observations regarding convex conjugates:
    \begin{eqnarray*}
        g(x) = f(\alpha x) & \Rightarrow & g^*(\beta) = f^*(\alpha^{-1}\beta)\\
        g(x) = \alpha f(x) & \Rightarrow & g^*(\beta) = \alpha f^*(\alpha^{-1}\beta)
    \end{eqnarray*} 
    These two rules follow directly from taking suprema on both sides of these equalities:
    \begin{eqnarray*}
        \innprod{\beta}{x} - f(\alpha x) & = & \innprod{\alpha^{-1}\beta}{\alpha x} - f(\alpha x) \\
        \innprod{\beta}{x} - \alpha f(x) & = & \alpha [\innprod{\alpha^{-1}\beta}{ x} - f( x) ]
    \end{eqnarray*} 
    
    So taking the convex conjugate on both sides of $f(\alpha x) = \alpha^r f(x)$ yields:
    
    \begin{equation}\label{eqn:convex_rule}
        f^*(\alpha^{-1}\beta) = \alpha^r f^*(\alpha^{-r}\beta)
    \end{equation}
    
    Let's do two change of variables, $v = \alpha^{-1}\beta$ and $\gamma = \alpha ^{1-r}$. Note that $\gamma v = \alpha^{-r}\beta$. Plugging all these into \Cref{eqn:convex_rule} changes variables from $\alpha,\beta$ to $\gamma,v$, as illustrated below:
    
    \begin{eqnarray*}
        f^*(v) & = & \alpha^r f^*(\gamma v)\\
        f^*(\gamma v) & = & \alpha^{-r} f^*(v)\\
        & = & \gamma^{\frac{r}{r-1}} f^*(v) = \gamma ^ s f^*(v)
    \end{eqnarray*}

\end{proof}


\subsection*{Proof of \Cref{claim:f}}\label{pf:{claim:f}}
\begin{claim-non}

    For any abstract norm $\|\cdot\|$ on $\Rd$ and any $r>1$
    $$f(x) = \frac{1}{r} \|x\|^r$$ satisfies the conditions of Theorem \ref{thm:convex}. Namely,
    \begin{itemize}
        \item $f(x)>0$ for all $x\neq0$;
        \item $f$ is convex over $\Rd$;
        \item $f$ is positive homogeneous of degree $r>1$.
    \end{itemize}
    Furthermore, let $\|\cdot\|_*$ be the dual norm of $\|\cdot\|$. $f^*$ has the following form:
    \[
    f^*(\beta) = \frac{1}{s} \|\beta\|_*^s.
    \]
    Here $1/r+1/s=1.$

\end{claim-non}

\begin{proof}
The conditions can be easily checked. We only prove the expression for $f^*$.

By definition, we have that $\innprod{x}{\beta} \leqslant||x|| \cdot ||y||_*$, so
$$\innprod{x}{\beta} - \frac{1}{r}||x||^r \leqslant||x||\cdot||\beta||_* - \frac{1}{r}||x||^r$$
for any $x\in\mathbb{R}^d$.  We treat the right-hand side as a function of $\|x\|$, which is upper-bounded by $\frac{1}{s}||\beta||_*^s$. We can rewrite the inequality as
$$\innprod{x}{\beta} - \frac{1}{r}||x||^r\leqslant\frac{1}{s}||\beta||_*^s$$

Since the left-hand side is an upper bound of $f^*(\beta)$, $f^*(\beta) \leqslant\frac{1}{s}||\beta||_*^s$.

Next, let $\tilde{x}\in\mathbb{R}^d$ be such that $\innprod{\tilde{x}}{\beta} = ||\tilde{x}||\cdot||\beta||_*$, appropriately scaled such that $||\tilde{x}||^r=||\beta||_*^s$.  Then, for this particular $\tilde{x}$,
$$\innprod{\tilde{x}}{\beta}-\frac{1}{r}||\tilde{x}||^r = \frac{1}{s}||\beta||_*^s$$

Therefore, we have the inequality
$$f^*(\beta)\geqslant \frac{1}{s}||\beta||_*^s$$

Since we now have inequalities in both directions, we must have that $f^*(\beta)=\frac{1}{s}||\beta||_*^s$, as desired.
\end{proof}